\documentclass[twoside]{article}

\usepackage[accepted]{aistats2020}

\usepackage{hyperref}       
\usepackage{url}            
\usepackage{booktabs}       

\usepackage{bm}
\usepackage{mathtools}
\usepackage{amsthm}
\usepackage{array}
\DeclareMathOperator*{\argmax}{arg\,max}

\usepackage[english]{babel}
\newtheorem{theorem}{Theorem}
\newtheorem{coro}{Corollary}
\newtheorem{defn}{Definition}
\usepackage[ruled,vlined]{algorithm2e}
\usepackage{todonotes}
\usepackage{amsmath}
\usepackage{mathrsfs}
\newcommand{\norm}[1]{\left\lVert#1\right\rVert}
\usepackage{amssymb}
\usepackage{soul}

\usepackage[round]{natbib}
\renewcommand{\bibname}{References}

\hypersetup{draft}

\begin{document}

\twocolumn[

\aistatstitle{(Updated) Why Non-myopic Bayesian Optimization is Promising and How Far Should We Look-ahead? A Study via Rollout}

\aistatsauthor{Xubo Yue \And Raed Al Kontar}

\aistatsaddress{University of Michigan \And University of Michigan}

]

\begin{abstract}
Lookahead, also known as non-myopic, Bayesian optimization (BO) aims to find optimal sampling policies through solving a dynamic program (DP) that maximizes a long-term reward over a rolling horizon. Though promising, lookahead BO faces the risk of error propagation through its increased dependence on a possibly mis-specified model. In this work we focus on the rollout approximation for solving the intractable DP. We first prove the improving nature of rollout in tackling lookahead BO and provide a sufficient condition for the used heuristic to be rollout improving. We then provide both a theoretical and practical guideline to decide on the rolling horizon stagewise. This guideline is built on quantifying the negative effect of a mis-specified model. To illustrate our idea, we provide case studies on both single and multi-information source BO. Empirical results show the advantageous properties of our method over several myopic and non-myopic BO algorithms. 
\end{abstract}

\section{Introduction}
\label{sec:intro}

Bayesian optimization is a popular technique to optimize an unknown and expensive-to-evaluate objective function through sequential sampling strategies. Traditionally BO has focused on myopic (also referred to as greedy) algorithms, where sampling points are decided based on a one-step lookahead utility function, oblivious to how this design will affect the future steps of the optimization and the remaining budget. 


Recently, motivated by reinforcement learning, there have been attempts to extend greedy BO methods into multi-step lookahead algorithms that maximize a reward over a rolling horizon. Though it seems promising to look further into the future, this approach might sabotage performance due to accumulated errors and increased dependence on a possibly mis-specified model. This raises the question: is a practical implementation of non-myopic approaches indeed useful? Although we cannot give a universal answer, we can shed light on a specific class of non-myopia: rollout dynamic programming.


Rollout is a sub-optimal approximation algorithm to sequentially solve intractable dynamic programming problems. It utilizes problem-dependent heuristics to approximate the future reward using simulations
over several future steps (i.e., the rolling horizon). Indeed, rollout has been successfully applied to the non-myopic BO scenario \citep{lam2016bayesian, lam2017lookahead}. Yet, rollout still faces two challenges: theoretical justification/guarantees and error propagation as errors from a mis-specified model will accumulate as we look further into the future. These challenges raise the question whether long term planning in BO is necessary. 

In this work, we first provide theoretical justification for rollout in BO settings. Specifically we show that under the class of sequentially improving heuristics, rollout is guaranteed to outperform its myopic counterpart. We then provide a guideline to carefully choose a rolling horizon at each stage of the discounted DP. Based on these facts, we argue that a short horizon is beneficial and also computationally economical. Therefore, using non-myopia is promising and deserves further research attention.

We organize the remaining paper as follows. In Sec. \ref{sec:BO_DP}, we briefly review BO, DP and rollout. We then prove the performance guarantee of rollout in Sec. \ref{sec:improve} and give a practical guideline on choosing the rolling horizon in Sec. \ref{sec:h}. In Sec. \ref{sec:case}, we provide case studies to evidence our theoretical argument. Detailed literature review can be found in Sec. \ref{sec:literature}. Some algorithmic details and multi-information source BO are included in the supplementary material.

\section{Background}
\label{sec:BO_DP}
In this section, we provide the problem description and a brief review on the technical background needed for this paper.
\subsection{Bayesian Optimization}
\label{sec:BO}
Let $f:\bm{\mathcal{X}}\to\mathbb{R}$ be an objective function which is expensive to evaluate. We consider the optimization problem:
\begin{equation}
    \bm{x}^*=\argmax_{\bm{x}\in\bm{\mathcal{X}}} f(\bm{x}),
\end{equation}
where $\bm{x}$ is a $d$-dimensional input/design vector and $\bm{\mathcal{X}}$ is a compact set in $\mathbb{R}^d$. Given limited budget $B$, BO aims to search for the optimal $\bm{x}^*$ by iteratively updating a surrogate model of $f(\bm{x})$, where this surrogate is used to find the next design point. Typically, in BO, the surrogate model is a Gaussian process ($\mathcal{GP}$), due to its Bayesian interpretation and uncertainty quantification capability (see \cite{rasmussen2003gaussian} for more information).

Without loss of generality, suppose we can sample $N$ design points. Given the current data $\bm{D}_k$, $k\in\{1,\ldots,N\}$, BO aims to determine the next informative sampling point $\bm{x}_{k+1}$ by solving the auxiliary problem:
\begin{equation}
    \label{eq:acq} 
    \bm{x}_{k+1}=\argmax_{\bm{x}\in\bm{\mathcal{X}}} Q_k(\bm{x};\bm{D}_k),
\end{equation}
where $Q_k$ is a acquisition/utility function that only involves evaluating the surrogate and not the expensive objective function $f$. Typically, evaluation of an acquisition function is relatively cheap. The rationale is to seek design points that produce maximum increment in the objective function. After Eq. \eqref{eq:acq} is solved, we sample at location $\bm{x}_{k+1}$ and observe the output $y_{k+1}$. The iterative algorithm proceeds by augmenting the current training data $\bm{D}_k$ with a new observation to obtain $\bm{D}_{k+1}=\bm{D}_k\cup\{(\bm{x}_{k+1},y_{k+1})\}$.  Popular choices of acquisition functions are entropy search (ES) \citep{hennig2012entropy}, predictive entropy search (PES) \citep{hernandez2014predictive} and expected improvement (EI) \citep{lam2016bayesian}. All aforementioned functions exploit myopic strategies and ignore the future information. 

\subsection{Dynamic Programming}
\label{sec:DP}
Lookahead BO can be directly viewed as an instance of DP. In such settings the non-myopic acquisition function quantifies rewards over future steps. Due to limited budge or sampling capacity, we consider a finite $N$-stage DP formulation. Denote by $k\in\{1,..., N\}$. At each stage $k$, define the state space as $\mathcal{S}_k=(\bm{\mathcal{X}}\times\mathbb{R})$ and denote by dataset $\bm{D}_k\coloneqq s_k\in\mathcal{S}_k$ the current state, where $s_k$ is the state in the state space $\mathcal{S}_k$. A policy $\bm{\pi}=\{\pi_1,\ldots,\pi_N\}\in\bm{\Pi}$ is a sequence of rules (i.e., sampling actions) $\pi_k$ mapping the state space $\mathcal{S}_k$ to the design space $\bm{\mathcal{X}}$, where $\bm{\Pi}$ is a policy space. We use $\pi^{\bm{\pi}}_k$ to emphasize the $k^{th}$ rule under policy $\bm{\pi}$. Let $\pi_k(\bm{D}_k)=\bm{x}_{k+1}$. 

Now denote by $r_k:\mathcal{S}_k\times\bm{\mathcal{X}}\to\mathbb{R}$ the reward function at stage $k$. The reward function $r_k(\bm{D}_k,\bm{x}_{k+1})$ quantifies the benefits of sampling at location $\bm{x}_{k+1}$ given the current dataset $\bm{D}_k$. For example, one popular choice of the reward function is the expected improvement function \citep{frazier2018tutorial}.

As there is no sampling action at the end-stage (i.e., $k=N+1$), we define the end-stage reward as $r_{N+1}:\mathcal{S}_{N+1}\to\mathbb{R}$. As a result, the discounted expected cumulative reward of a finite $N$-step horizon under policy $\bm{\pi}$ given initial dataset $\bm{D}_1$ can be expressed as $ R^{\bm{\pi}}(\bm{D}_1)=$
\begin{equation}
\label{eq:reward}
   \mathbb{E}\bigg[\sum_{k=1}^{N} \alpha^{k-1} r_k(\bm{D}_k,\bm{x}_{k+1})+\alpha^N r_{N+1}(\bm{D}_{N+1})\bigg],
\end{equation}
where $\alpha\in[0,1]$ is the discount factor. The discount factor plays an important role in this setting, as it controls the effect of error propagation. \emph{In the greedy algorithm, we have $\alpha=0$}. In the policy space $\bm{\Pi}$, we are interested in the optimal policy $\bm{\pi}^*\in\bm{\Pi}$ which maximizes Eq. \eqref{eq:reward}. Specifically,
\begin{equation}
\label{eq:reward2}
    R^{\bm{\pi}^*}(\bm{D}_1)\coloneqq\max_{\bm{\pi}\in\bm{\Pi}}R^{\bm{\pi}}(\bm{D}_1).
\end{equation}
Based on the Bellman optimality equation, we can then formulate \eqref{eq:reward} and \eqref{eq:reward2} as a recursive DP:
\begin{equation}
\label{eq:DP}
\begin{split}
    &R_k(\bm{D}_k)=\max_{\bm{x}_{k+1}\in\bm{\mathcal{X}}}\mathbb{E}[r_k(\bm{D}_k,\bm{x}_{k+1})+\alpha R_{k+1}(\bm{D}_{k+1})],\\
    &R_{N+1}(\bm{D}_{N+1})=r_{N+1}(\bm{D}_{N+1}),
\end{split}
\end{equation}
where $R_k(\cdot)$ is known as the reward-to-go function. Without loss of generality, we set $R_{N+1}(\bm{D}_{N+1})=0$. Therefore, at stage $k$, the next sampling location is decided by maximizing $\mathbb{E}[r_k(\bm{D}_k,\bm{x}_{k+1})+\alpha R_{k+1}(\bm{D}_{k+1})]$. In the context of BO, we can naturally set the acquisition function to be $$Q_k(\bm{x};\bm{D}_k)=\mathbb{E}[r_k(\bm{D}_k,\bm{x}_{k+1})+\alpha R_{k+1}(\bm{D}_{k+1})]$$. 


\subsection{Rollout}
\label{sec:rollout}
The DP formulation in Sec. \ref{sec:DP} is subject to a huge computational burden and curse of dimensionality due to the uncountable state and action space. Furthermore, the formulation assumes that data in the last step is available and computes the acquisition function in a backward manner, which is impractical in BO. In order to solve the intractable DP, an approximate dynamic programming (ADP) approach - rollout \citep{bertsekas1995dynamic} has been proposed. Rollout has recently enjoyed success across a variety of domains as it builds on several heuristic rules $\tilde{\pi}_k, \forall k$ (details later) \citep{lam2016bayesian} and is efficient for large-scale and finite-horizon DP problems. Instead of solving DP in a backward manner, rollout solves DP in a forward manner. Here we briefly describe the rollout algorithm. We first define a key component - the rolling horizon $h (h\geq 1)$ and let $\tilde{N}=\min\{k+h,N\}$. At stage $k$, rollout first decides a sampling location $\tilde{\bm{x}}_{k+1}$ using a heuristic policy and collect a simulated $\tilde{y}_{k+1}$ from the surrogate model. For example, the heuristic policy and the simulated output $\tilde{y}_{k+1}$ can be the sampling action and the output generated from the expected improvement function. Afterward, we create a simulated dataset $\tilde{\bm{D}}_{k+1}=\bm{D}_k\cup\{\tilde{\bm{x}}_{k+1},\tilde{y}_{k+1}\}$. Based on this simulated dataset, one can use a similar aforementioned procedure to collect $\{\tilde{\bm{x}}_{k+2},\tilde{y}_{k+2}\},\ldots,\{\tilde{\bm{x}}_{k+h},\tilde{y}_{k+h}\}$. Using the simulated dataset $\tilde{\bm{D}}_{k+1},\ldots,\tilde{\bm{D}}_{k+h}$, one can further quantify rewards $r_k,\ldots,r_{k+h}$. As a result, we select the optimal sampling location $\bm{x}_{k+1}$ that maximizes the accumulated reward over a rolling horizon. Mathematically, we are optimizing the following reward-to-go functions:
\begin{equation}
\label{eq:rollout}
    \begin{split}
        H_k(\bm{D}_k)&=\mathbb{E}[r_k(\bm{D}_k,\tilde\pi_k(\bm{D}_k))+\alpha H_{k+1}(\tilde{\bm{D}}_{k+1})],\\
        H_{\tilde N}(\bm{D}_{\tilde N})&=r_{\tilde N}(\tilde{\bm{D}}_{\tilde N},\tilde\pi_{\tilde N}),
    \end{split}
\end{equation}
where $\tilde\pi_k$ is the heuristic rule at every iteration $k\in[\tilde N]=\{1,\ldots,\tilde N\}$ such that $\tilde{\pi}_k(\bm{D}_k)=\tilde{\bm{x}}_{k+1}$. For example, if one uses the EI acquisition function, the heuristic rule is `` sampling at location $\tilde{\bm{x}}_{k+1}$ that provides the maximal improvement". Here note that the heuristic rule only samples at the location that maximizes the current acquisition function and ignores the long-term reward. However, this does not indicates rollout is myopic. In fact, $\tilde{\bm{x}}_{k+1}$ is a simulated sampling location that will be used to create simulated datasets $\tilde{\bm{D}}_{k+1},\ldots,\tilde{\bm{D}}_{k+h}$. The final decision on the sampling location $\bm{x}_{k+1}$ (without the tilde notation) is selected to maximize the accumulated reward over a rolling horizon. In essence, this feature makes rollout a non-myopic algorithm.  

Eq. \ref{eq:DP} and Eq. \ref{eq:rollout} have a key difference: the former one has a maximization operator. In Eq. \ref{eq:DP}, to compute the optimal sampling location at stage $k$, one needs to know the optimal $R_{k+1}$ at stage $k+1$. This is apparently infeasible as we do not have any information about $R_{k+1},\bm{D}_{k+1}$ at stage $k$. Eq. \ref{eq:rollout}, on the other hand, circumvents this situation. It removes the maximization operator so that the sampling action at stage $k$ is independent of the future. Therefore, the intractable acquisition function $Q_k(\bm{x};\bm{D}_k)$ from Sec. \ref{sec:DP} can be approximated by an approximate acquisition function $\tilde{Q}_k(\bm{x};\bm{D}_k)\coloneqq H_k(\bm{D}_k)$. At the end stage, we define policy $\tilde\pi_{\tilde N}$ such that $\bm{x}^*=\argmax_{\bm{x}\in\mathcal{X}}\mu^{\tilde N}_0(\bm{x})$, where $\mu^{\tilde N}_0$ is the updated mean function from data $\bm{D}_{\tilde N}$. For instance, if the surrogate model is a $\mathcal{GP}$, then $\bm{D}_{\tilde N}$ is the posterior mean of a $\mathcal{GP}$ \citep{rasmussen2003gaussian}. 





\section{Rollout Performance Guarantees}
\label{sec:improve}
Without loss of generality and for the sake of neatness, we omit the discount factor and assume $\alpha=1$. Given a state $s$, an algorithm $\mathcal{H}(s)$ is a method to select a sequence of feasible rules $\{\pi_k\}_{k=1}^N$ and policy $\bm{\pi}_{\mathcal{H}(s)}$ which generates states $\{s_k\}_{k=1}^N$. Now, to establish theoretical guarantees, we first provide the following definitions \citep{bertsekas1997rollout, goodson2017rollout}. 
\begin{defn} 
Consider a maximization problem. Algorithm $\mathcal{H}$ is said to be sequentially consistent if for every state $s_k\neq s_N$, whenever $\mathcal{H}$ generates the state path $(s_k,s_{k+1}\ldots,s_N)$ starting at state $s_k$, $\mathcal{H}$ also generates the path $(s_{k+1},\ldots,s_N)$ starting at state $s_{k+1}$.
\end{defn}
In the context of DP, let $s\in\mathcal{S}$ and let $s'$ be a state on a path generated by policy $\bm{\pi}$ using algorithm $\mathcal{H}(s)$. Denote this policy as $\bm{\pi}_{\mathcal{H}(s)}$. Consequently, sequential consistency can equivalently be defined as

\begin{defn}
Consider a maximization problem. Algorithm $\mathcal{H}$ is said to be sequentially consistent if $\forall s$ and subsequent $s'$, we have
\begin{equation}
    \begin{split}
        (\pi^{\bm{\pi}_{\mathcal{H}(s)}}_k,& \pi^{\bm{\pi}_{\mathcal{H}(s)}}_{k+1}, \ldots, \pi^{\bm{\pi}_{\mathcal{H}(s)}}_N)\\
        &=(\pi^{\bm{\pi}_{\mathcal{H}(s')}}_k, \pi^{\bm{\pi}_{\mathcal{H}(s')}}_{k+1}, \ldots, \pi^{\bm{\pi}_{\mathcal{H}(s')}}_N).
    \end{split}
\end{equation}
\end{defn}

\textbf{Definitions 1} implies that an algorithm is sequentially consistent if it produces the same subsequent states when started at any intermediate state of a path that it generates. Equivalently, by \textbf{Definition 2}, the algorithm will generate the same subsequent rules $(\pi_k,\ldots,\pi_N)$.


Now, consider a probability space $(\Omega,\mathcal{F},P)$. Define $\sigma(s)$ as the sub $\sigma$-algebra generated by the state $s$, then we have the following definition: 
\begin{defn}
Consider a maximization problem. Algorithm $\mathcal{H}$ is said to be sequentially improving if for every state $s_k\neq s_N$, whenever $\mathcal{H}$ generates the path $(s_k,s_{k+1},\ldots,s_N)$ starting at state $s_k$, the following property will hold
\begin{equation}
    \label{improve1}
    \begin{split}
        \mathbb{E}\big[\sum_{\ell=k}^N r_\ell&(s_\ell,\pi^{\bm{\pi}_{\mathcal{H}(s_k)}}_\ell(s_\ell))|\sigma(s_{k+1})\big]\\
        &\leq\mathbb{E}\big[\sum_{\ell=k}^N r_\ell(s_\ell,\pi^{\bm{\pi}_{\mathcal{H}(s_{k+1})}}_\ell(s_\ell))|\sigma(s_{k+1})\big].
    \end{split}
\end{equation}
\end{defn}
It directly follows that if $\mathcal{H}$ is sequentially consistent, then the equality will hold in \eqref{improve1}. Therefore, a sequentially consistent algorithm is also sequentially improving. However, the converse is not true. Next, we will present our theorem about \textit{rollout improving}.

\begin{theorem}
    \label{improve2}
    The sequentially improving algorithm $\mathcal{H}$ is also rollout improving. Formally, given the rollout policy $\bm{\pi}$, we have the following property
\begin{equation}
    \begin{split}
    &\mathbb{E}\big[\sum_{\ell=k}^N r_\ell(s_\ell,\pi^{\bm{\pi}_{\mathcal{H}(s_k)}}_\ell(s_\ell))|\sigma(s_{k})\big] \\
    &\leq \mathbb{E}\big[\sum_{\ell=k}^N r_\ell(s_\ell,\pi^{\bm{\pi}}_\ell(s_\ell))|\sigma(s_k)\big].
    \end{split}
\end{equation}
\end{theorem}
\begin{proof}
We will prove this theorem by mathematical induction. When $\ell=N$, this statement is trivial. Now assume this statement holds for $\ell=k+1,\ldots,N-1$. Then, when $\ell=k$, define $\sigma(s_k)$ as the sub $\sigma$-algebra generated by state $s_k$. Since each subsequent state $s_{k+1}$ is an augmented $s_k$, we have $\sigma(s_k)\subseteq\sigma(s_{k+1})\subseteq\mathcal{F}$. By the law of total expectation, we have
\begin{equation}
    \begin{split}
        &\mathbb{E}\big[\sum_{\ell=k}^N r_\ell(s_\ell,\pi^{\bm{\pi}_{\mathcal{H}(s_k)}}_\ell(s_\ell))|\sigma(s_k)\big]\\
        &=\mathbb{E}\bigg[\mathbb{E}\big[\sum_{\ell=k}^N r_\ell(s_\ell,\pi^{\bm{\pi}_{\mathcal{H}(s_k)}}_\ell(s_\ell))|\sigma(s_{k+1})\big]\bigg|\sigma(s_k)\bigg].
    \end{split}
\end{equation}
By assumption, since the algorithm is sequentially improving, we have
\begin{equation}
    \begin{split}
        &\mathbb{E}\bigg[\mathbb{E}\big[\sum_{\ell=k}^N r_\ell(s_\ell,\pi^{\bm{\pi}_{\mathcal{H}(s_k)}}_\ell(s_\ell))|\sigma(s_{k+1})\big]\bigg|\sigma(s_k)\bigg]\\
        &\leq\mathbb{E}\bigg[\mathbb{E}\big[\sum_{\ell=k}^N r_\ell(s_\ell,\pi^{\bm{\pi}_{\mathcal{H}(s_{k+1})}}_\ell(s_\ell))|\sigma(s_{k+1})\big]\bigg|\sigma(s_k)\bigg]\\
        &=\mathbb{E}\bigg[r_k(s_k,\pi^{\bm{\pi}_{\mathcal{H}(s_{k+1})}}_k(s_k))+\\
        &\quad \mathbb{E}\big[\sum_{\ell={k+1}}^N r_\ell(s_\ell,\pi^{\bm{\pi}_{\mathcal{H}(s_{k+1})}}_\ell(s_\ell))|\sigma(s_{k+1})\big]\bigg|\sigma(s_k)\bigg]\\
        &\leq\max_{\pi} \mathbb{E}\bigg[r_k(s_k,\pi)+\\
        &\quad \mathbb{E}\big[\sum_{\ell={k+1}}^N r_\ell(s_\ell,\pi^{\bm{\pi}_{\mathcal{H}(s_{k+1})}}_\ell(s_\ell))|\sigma(s_{k+1})\big]\bigg|\sigma(s_k)\bigg]\\
        &=\mathbb{E}\bigg[r_k(s_k,\pi_k^{\bm{\pi}}(s_k))+\\
        &\quad \mathbb{E}\big[\sum_{\ell={k+1}}^N r_\ell(s_\ell,\pi^{\bm{\pi}_{\mathcal{H}(s_{k+1})}}_\ell(s_\ell))|\sigma(s_{k+1})\big]\bigg|\sigma(s_k)\bigg].
    \end{split}
\end{equation}
The last equality follows the definition of the rollout algorithm. The rest of the proof is completed by the induction hypothesis.
\end{proof}
\textbf{Theorem \ref{improve2}} shows that the rollout approach is guaranteed to perform better than its myopic counterpart under the same base heuristic rules. Intuitively, when rollout generates a path, it exploits the base heuristic to generate a collection of other paths and picks up the best one. In the next section, we will provide a guideline on choosing a sequence of feasible rolling horizons.

\section{Deciding on the Rolling Horizon}
\label{sec:h}
One interesting question remains: how to decide the rolling horizon $h$? In most of the literature, $h$ is chosen to be a fixed value within 2 and 5 \citep{lam2017lookahead, ulmer2018offline} in order to alleviate computational burden. Though those choices give very promising results, those decisions are very subjective. Fortunately, based on the rollout theory \citep{bertsekas1997rollout, bertsekas2005rollout}, we can provide a practical guideline to select a stagewise feasible $h$. The big picture is as follows: we quantitatively obtain the benefits of rollout given a modeling error and discount factor, we then compare this long-term discounted benefit with the reward from the greedy algorithm counterpart and decide a feasible rolling horizon accordingly. We provide a detailed argument below. 

At each stage $k$, define a profit function $g_k:\mathbb{Z}^+\to\mathbb{R}^+$ related to the rolling horizon $h$ such that
\begin{equation}
    \label{4:profit}
    g_k(h)=\sum_{i=k+1}^{k+h}\phi(i-k),
\end{equation}
where $\phi(\cdot)$ is a non-negative function. The rolling profit function can be viewed as the total benefits incurred when choosing a rolling horizon $h$ at stage $k$. For example, at stage $k$, $g_k(1)$ is the reward function using rolling horizon $1$ and $g_k(2)$ is the accumulated reward function using $h=2$. Although long horizons provide more future information, it is not guaranteed to be helpful. In practice we are running the risk of model mis-specification due to modeling the objective function using a $\mathcal{GP}$ and then using this surrogate to simulate scenarios over future steps. Therefore, a larger rolling horizon implies an increased dependence on a possibly erroneous model which might in turn cause adverse effects compared to myopic algorithms where errors accumulate only from a one-step lookahead. However, if we  can arbitrarily quantify the error from mis-specified model, then we can utilize the rollout improving nature and accordingly decide on the feasible rolling horizon.



In order to quantify the aforementioned error, we define an error function $\mathcal{E}(\bm{x})$ bounded by a constant $\bar e_k$. The $\mathcal{E}(\cdot)$ is a metric to quantify the negative effect from model mis-specification. In the next section, we will provide an error bound on $\mathcal{GP}$ prediction and use this error bound as an error. 


\subsection{Error Bound on the GP}
\label{sec:robust}
The recent work of \cite{wang2019prediction} sheds light on the model mis-specification issue.
\begin{coro}
    \label{robust}
    \citep{wang2019prediction} Assume a $\mathcal{GP}$ with zero mean and stationary convariance function. Then, under some regularity conditions, the interpolation error is (non-asymptotically)
    \begin{equation*}
        \sup_{\bm{x}}|y(\bm{x})-\hat f(\bm{x})|\leq K\sigma^2P_X\sqrt{\log(\frac{e}{P_X})}+u, 
    \end{equation*}
    with probability $1-\delta$, where $\delta$ is a function of $u,P_X$ and $\sigma^2$, $y(\bm{x})$ is the true output at input $\bm{x}$ and $P_X=\sqrt{1- K(\bm{x}^*,\bm{x}) K(\bm{x},\bm{x})^{-1}K(\bm{x},\bm{x}^*)}$ is a power function with mis-specified covariance function at observation $\bm{x}$, $K$ and $u$ are some constants and $\sigma^2$ is the variance parameter.
\end{coro}
In practice, if one uses the Mat\'ern kernel with smooth parameter $v$ \citep{rasmussen2003gaussian}, then the upper bound can be approximated by \citep{wang2019prediction}
\begin{align*}
    F_X^v\sqrt{\log(\frac{1}{F_X})},
\end{align*}
where $F_{X}=\max_{\bm{x}\in\mathcal{X}}\min_{\bm{x}'\in X}\norm{\bm{x}-\bm{x}'}$ and $X$ is the current dataset that contains all design points \citep{johnson1990minimax}. In this paper, we only focus on the $\mathcal{GP}$ with Mat\'ern kernel as it is robust to model mis-specifications \citep{wang2019prediction, burt2019rates}.

Given this result, at each stage $k$, we can define $\mathcal{E}(\bm{x})=|y(\bm{x})-\hat f(\bm{x})|\leq\sup_{\bm{x}}|y(\bm{x})-\hat f(\bm{x})|=\bar e_k$. One regularity condition in \textbf{Corollary \ref{robust}} is that the mis-specified kernel is no smoother than the true kernel. The mat\'ern kernel is one of the perfect candidates to this requirement \citep{burt2019rates}. In the next section we use this error function to find feasible $h$ stagewise. We note that the proposed framework can be substituted with a different error function. For example, one can use the standard deviation obtained from the surrogate model. In this paper, we focus on $\mathcal{GP}$s as they are the most commonly used surrogate in BO.



\subsection{Deciding on the Rolling Horizon}
\label{sec:optimal_h}
In this section, we provide our main theorem on deciding $h$ stagewise. Define a function
\begin{equation}
    \label{eq:optimal_A}
    \begin{split}
        A_k(h,s_k)&=\max\bigg\{\max_{\bm{x}_{k+1}}\mathbb{E}\big[r_{k+1}(s_{k+1},\pi^{\bm{\pi}_{\mathcal{H}(s_k)}}_{k+1}(s_{k+1}))\big],\\ &\max_{\bm{x}_{k+1}}\mathbb{E}\big[g_k(1)+\alpha\tilde H(s_{k+1}) \big]\bigg\},
    \end{split}
\end{equation}
where $\tilde H(s_{k+1})=g_{k+1}(2)-g_{k+1}(1)-\mathcal{E}(\bm{x}_{k+2})+\alpha\tilde H(s_{k+2}) $ is a modified rollout reward function to quantify both profit and error effects. Note that we do not consider $\mathcal{E}(\bm{x}_{k+1})$ since it is shared by both algorithms when the rolling horizon is 1. $A_k$ denotes the optimal reward from stage $k$ to $N$, given the current state $s_k$ and an unknown rolling horizon $h$. Eq. \eqref{eq:optimal_A} returns the maximum element between two values: the first one is the reward when we consider a greedy algorithm (i.e., $h=1$) and the second one is the reward when we consider the rollout algorithm given a certain error function. Based on Eq. \eqref{eq:optimal_A}, we can obtain the following theorem.
\begin{theorem}
    \label{thm:h_indirect}
    The set of feasible rolling horizons at stage $k$ is defined as
    \begin{equation}
        \begin{split}
            h^*&=\bigg\{2\leq h\leq N|A_k(h,s_k)\\
            &=\max_{\bm{x}_{k+1}}\mathbb{E}\big[g_k(1)+\alpha\tilde H(s_{k+1})\big]\bigg\}.
        \end{split}
    \end{equation}
\end{theorem}
\textbf{Theorem \ref{thm:h_indirect}} implies that for any $h$ within this set the rollout is more beneficial than a greedy algorithm. In other words, the benefits gained from looking further ahead outweigh that of the error effects. However, calculating $h^*$ is hard to implement in practice. In the next theorem, we will provide an equivalent but more practical equation. 

\begin{theorem}
    \label{thm:h}
    \textbf{The Rolling Horizon Theorem} Given a constant $\bar e_k$ on the error function $\mathcal{E}(\bm{x})$ and the profit function defined in Eq. \eqref{4:profit}. The feasible rolling horizon at stage $k$ is defined as
    \begin{equation*}
        h^*=\bigg\{j\in\mathbb{Z}:\sum_{i=2}^j\alpha^{i-2}\phi(i)>\bar e_k\frac{1-\alpha^{N-k}}{1-\alpha}\bigg\}.
    \end{equation*}
\end{theorem}
\begin{proof}
    Based on \textbf{Theorem \ref{thm:h_indirect}} and Eq. \eqref{eq:optimal_A}, it is equivalent to consider 
    \begin{equation}
        \label{eq:thm3}
        \begin{split}
            &\max_{\bm{x}_{k+1}}\mathbb{E}\big[g_k(1)+\alpha\tilde H(s_{k+1})\big] \geq \\
            &\quad\max_{\bm{x}_{k+1}}\mathbb{E}\big[r_{k+1}(s_{k+1},\pi^{\bm{\pi}_{\mathcal{H}(s_k)}}_{k+1}(s_{k+1}))\big].
        \end{split}
    \end{equation}
    By definition, $\max_{\bm{x}_{k+1}}\mathbb{E}\big[r_{k+1}(s_{k+1},\pi^{\bm{\pi}_{\mathcal{H}(s_k)}}_{k+1}(s_{k+1}))\geq g_k(1)$. Since $\tilde H(s_{k+1})\geq-\sum_{i=k+2}^{N+1}\alpha^{i-k-2}\mathcal{E}(\bm{x}_{i})\geq-\bar e_k\frac{1-\alpha^{N-k-1}}{1-\alpha}$, we have $\tilde H(s_{k+1})=g_{k+1}(2)-g_{k+1}(1)-\mathcal{E}(\bm{x}_{k+2})+\alpha\tilde H(s_{k+2})\geq\phi(2)-\bar e_k-\bar e_k\frac{1-\alpha^{N-k-1}}{1-\alpha}\alpha$. Therefore, Eq. \eqref{eq:thm3} can be simplified as $\phi(2)\geq\frac{1-\alpha^{N-k}}{1-\alpha}\bar e_k$. The remaining part can be obtained by induction. 
\end{proof}
In practice, we can pick up the minimal $j$ from the set $h^*$. We can also set an upper bound on the rolling horizon. Denote by it $\bar h$. In the \textbf{Theorem \ref{thm:h}}, if we could not find feasible $h$ till $j=\bar h$, we stop searching and use $h=1$ at the current stage.

\section{Case Study}
\label{sec:case}
In this section we provide two case studies to evidence our theoretical arguments. We use the well-known knowledge gradient (KG, see Appendix) \citep{poloczek2017multi} as the base algorithm in our rollout algorithm. More specifically, we will use sampling actions generated by KG as heuristic rules. We then show that KG is both sequentially consistent and improving, and thus it is rollout improving as shown in \textbf{Theorem \ref{improve2}}. We then illustrate that, under \textbf{Theorem \ref{thm:h}} and through carefully choosing the rolling horizon, non-myopic BO has strong advantages over greedy BO. Our algorithm is tested for both single source and multi-information source BO (misoBO).  In the misoBO setting, we sample from auxiliary information sources to make inference. Here we note that the details for misoBO are deferred to the appendix due to space limitation and similar conclusions to that of single source BO.

\subsection{Setting}
We use the same setting in Sec. \ref{sec:BO}. Specifically, when sampling from original function at input $\bm{x}$, we observe an outcome $y(\bm{x})$. We assume the observation $y(\bm{x})$ is normally distributed with mean $f(\bm{x})$ and variance $\sigma^2(\bm{x})$. For the purpose of robustness, we assume that the covariance belongs to some non-smooth parametric family. Specifically, we will use the Mat\'ern kernel. Parameters are estimated using maximum likelihood estimation (MLE).

\subsection{Algorithm}
We utilize the non-greedy acquisition function $Q_k$ defined in Sec. \ref{sec:DP}. This acquisition function considers far horizon planning and is given by the DP formulation. Specifically, the original acquisition function can be defined as $Q_k(\bm{x};\bm{\mathcal{D}}_k)=\mathbb{E}\big[r_k(\bm{D}_k,\bm{x}_{k+1})+\alpha R_{k+1}(\bm{D}_{k+1}) \big]$. This is solved by the rollout with KG as the base heuristic. We denote our algorithm as DP-singleBO. The general procedure for DP-singleBO is listed in Algorithm \ref{algo:1}. We also extend this
algorithm to the multi-information source scenario and denote it as DP-misoBO (see Appendix).

\begin{algorithm}[!htb]
    \label{algo:1}
	\SetAlgoLined
	\KwData{Initial data $\bm{D}_1$, budget $B$ and query cost $c$, number of remaining evaluations $N$, bound $\bar h$.}
	\KwResult{Data $\bm{D}_N$, optimal value $f^{\bm{D}_{N}}_{max}$, Gap $G$.}
	Fit $\mathcal{GP}$ to data $\bm{D}_1$ and obtain initial optimal value $f^{\bm{D}_1}_{max}$\;
	\For{$k=1:N$}{
		\eIf{ $B < c$ }{
			Directly return $\bm{D}_k$ as $\bm{D}_N$\;
			STOP\;
		}{
			Calculate $h$ in Sec. \ref{sec:optimal_h}\;
			Given $h$, select $\bm{x}_{k+1}=\argmax_{\bm{x}\in\bm{\mathcal{X}}} \tilde{Q}_k(\bm{x};\bm{D}_k)$ s.t. $c(\bm{x}_{k+1})\leq B$\;
			$B$ $\gets$ $B-c(\bm{x}_{k+1})$\;
		}
		Evaluate $f(\cdot)$ at $\bm{x}_{k+1}$ and obtain $y_{k+1}$\;
		Augment the dataset $\bm{D}_{k+1}=\bm{D}_k\cup\{(\bm{x}_{k+1},y_{k+1})\}$\;
		Fit $\mathcal{GP}$ to data $\bm{D}_{k+1}$\;
		$k\gets k+1$\;
	}
	Fit $\mathcal{GP}$ to data $\bm{D}_{N}$\;
	Obtain optimal value $f^{\bm{D}_{N}}_{max}$\;
	Calculate the Gap $G$ (See Eq. \eqref{last})\;
	Return $\bm{D}_N$, $f^{\bm{D}_{N}}_{max}$ and $G$.
	\caption{The Non-myopic Single Information
		Source BO Algorithm}
\end{algorithm}

\subsection{Guarantees}
In order to apply \textbf{Theorem \ref{improve2}}, we need to show that the heuristic greedy KG is \textit{sequentially consistent} and thus \textit{sequentially improving}.
\begin{coro}
    The KG algorithm is sequentially consistent and sequentially improving.
\end{coro}
\begin{proof}
Remember that state $s_k$ is the dataset $\bm{D}_k$. Assume KG algorithm starts at a state $s_k$ (i.e., current dataset $\bm{D}_k$). At each iteration of KG, given a path $(\bm{D}_k, \bm{D}_{k+1}, \ldots, \bm{D}_m)$ and $\bm{D}_m$ is not the state at the end, the next state $\bm{D}_{m+1}$ is obtained by solving the acquisition function of KG (see appendix) and augmenting $\bm{D}_m$ with $(\bm{x}_{m+1}, y_{m+1})$. If $\bm{D}_{m+1}$ is not the terminating state, the algorithm will then start with the path $(\bm{D}_k, \bm{D}_{k+1}, \ldots, \bm{D}_m,\bm{D}_{m+1}$). Otherwise, the algorithm will terminate with state $\bm{D}_{m+1}$ and $N=m+1$. Therefore, KG is sequentially consistent.  

Let $(\bm{D}_1, \bm{D}_2, \ldots, \bm{D}_k,\ldots, \bm{D}_N)$ be the path generated by the rollout starting from $\bm{D}_1$. Define $\sigma(s)$ as the sub $\sigma$-algebra generated by state $s$.  Since KG is sequentially consistent, we have
\begin{equation}
    \label{4:1:1}
    \begin{split}
        &\mathbb{E}\big[\sum_{\ell=k}^N r_\ell(s_\ell,\pi^{\bm{\pi}_{\mathcal{H}(s)}}_\ell(s_\ell))|\sigma(s')\big]\\
        &=\mathbb{E}\big[\sum_{\ell=k}^N r_\ell(s_\ell,\pi^{\bm{\pi}_{\mathcal{H}(s')}}_\ell(s_\ell))|\sigma(s')\big].
    \end{split}
\end{equation}
Therefore, KG is sequentially improving and we complete our proof.
\end{proof}

\subsection{Results}
\subsubsection{Performance Comparison}
In this section, we apply algorithms DP-singleBO and DP-misoBO to a variety of classical functions with a range of dimensions, support sets and information sources. We provide three information sources in this experiment: original objective function $y(\bm{x})$, biased source one $y(1,\bm{x})$ and biased source two $y(2,\bm{x})$. Following the setting from \cite{poloczek2017multi}, we define $y(1,\bm{x})=y(\bm{x})+2\sin(10x_1+5x_2)$ in the two dimensional space and $y(1,\bm{x})=y(\bm{x})+2\sin(10x_1+5x_2+3x_3)$ in the three dimensional space. We define $y(2,\bm{x})=y(\bm{x})+\delta(\bm{x})$, where $\delta(\bm{x})$ is simulated from $\mathcal{GP}$ with radial basis function (RBF) kernel with length-scale $l=1$, signal variance $\sigma^2_f=1$ and noise variance $\sigma^2_n=0.5$. The RBF kernel is defined as $K_{RBF}(\bm{x},\bm{x}')=\sigma^2_f\exp\{-\frac{1}{2l^2}\norm{\bm{x}-\bm{x}'}^2_2\}+\sigma^2_n\mathbb{I}(\bm{x},\bm{x}')$. See Table $\ref{function}$ for more information. For the Goldstein-price and Bohachevsky functions, we provide two biased sources and run DP-misoBO algorithm. For the Branin-Hoo, Six-Hump and Griewant, we run DP-singleBO algorithm. These objective functions have two notable challenges: (1) six-hump and Goldsterin-price have several local maxima; (2) Griewant function has a large design space. We benchmark our algorithms with several state-of-the-art techniques.
\begin{table}[h]
  \caption{Functions used in the experiment. More information about each function can be found at the open source library \url{http://www.sfu.ca/~ssurjano/optimization.html}.}
  \label{function}
  \centering
  \begin{tabular}{cc}
    \toprule
    Name & Function Domain\\
    \midrule
    Branin-Hoo & [-5, 10]$\times$ [0, 15]\\
    Six-hump Camel & [-3, 3] $\times$ [-2, 2]\\
    Goldstein-price & [-2, 2]$^2$\\
    Bohachevsky & [-100, 100]$^2$\\
    Griewant-$3$ & [-600, 600]$^3$\\ 
    \bottomrule
  \end{tabular}
  \vspace{-0.4 cm}
\end{table}

\paragraph{Experimental Details} To mitigate the negative effect of model mis-specification, we fit $\mathcal{GP}$s with the mat\'ern $p+\frac{1}{2}$ kernel and all hyperparameters are optimized by MLE. We set discount factor $\alpha$ to be 0.9. The optimal rolling horizon $h$ is calculated at each stage. The initial 9 sampling points are chosen by the fill distance design. For a fixed dimension $d$, we set an upper limit for sampling budget $B$ and only allow around $10d$ evaluations of each algorithm. We set $B=10d^2$, cost $c=5d$ and $c_i=d,\forall i$. For each algorithm, we conduct 30 experiments with different initial points. In Table \ref{result} we provide the testing results in terms of the mean and median of Gap, defined in Eq. \eqref{last}.


\begin{table*}[!htb]
  \caption{Mean and median Gap $G$ over 30 experiments with different initial points. The best result for each function is bolded. ``NA'' indicates not applicable. The discount factor is set to be 0.9.}
  \label{result}
  \centering
  \scriptsize
  \begin{tabular}{cccccccc}
    \toprule
    Function Name & & GLASSES & M-EI & misoKG & MPI & LCB & \textbf{DP-singleBO/DP-misoBO} \\
    \midrule
    Branin-Hoo & Mean & 0.761 & 0.837 & 0.819 & 0.606 & 0.612 &  \textbf{0.864} \\
               & Median & 0.814 & 0.856 & 0.827 & 0.614 & 0.637 & \textbf{0.889} \\
    Six-Hump Camel & Mean & 0.735 & 0.843 & 0.801 & 0.625 & 0.638 & \textbf{0.870}  \\
                   & Median & 0.793 & 0.843 & 0.810 & 0.593 & 0.638 & \textbf{0.866}  \\
    Goldstein-Price & Mean & NA & 0.831 & 0.811 & NA & NA & \textbf{0.867}  \\
                    & Median & NA & 0.837 & 0.846 & NA & NA & \textbf{0.857} \\
    Bohachevsky & Mean & NA & 0.806 & 0.786 & NA & NA & \textbf{0.872} \\
                & Median & NA & 0.821 & 0.820 & NA & NA & \textbf{0.870} \\
    Griewant-3 & Mean & 0.725 & 0.814 & 0.820 & 0.704 & 0.704 & \textbf{0.861}  \\ 
               & Median & 0.742 & 0.817 & 0.827 & 0.678 & 0.731 & \textbf{0.856} \\ 
    \bottomrule
  \end{tabular}
\end{table*}

\paragraph{Benchmark Models} There is a limited literature on the non-greedy BO. We will benchmark our model with the state-of-the-art GLASSES algorithm with fixed horizon, a DP-based algorithm using M-EI with fixed rolling horizon \citep{lam2015multifidelity}, Markov chain Monte Carlo (MCMC) based maximum probability of improvement (MPI) \citep{snoek2012practical}, MCMC based lower confidence bound (LCB) \citep{snoek2012practical} and the misoKG \citep{poloczek2017multi}. Note that GLASSES, MPI and LCB cannot be applied to the miso setting. We refer to section \ref{sec:literature} for more details on the benchmarked models. 


\paragraph{Performance} The performance is measured in terms of Gap $G$, which is a common metric in many BO literature \citep{huang2006global, gonzalez2016glasses, lam2016bayesian}. Specifically,
\begin{equation} 
\label{last}
    G\coloneqq\frac{f^{\bm{D}_1}_{max}-f^{\bm{D}_{N}}_{max}}{f^{\bm{D}_1}_{max}-f^*_{max}},
\end{equation}
where $f^{\bm{D}_1}_{max}$ and $f^{\bm{D}_{N}}_{max}$ are optimal values given the initial and augmented data at stage $N$ respectively and $f^*_{max}$ is the global maximum of the testing function. Table \ref{result} shows the comparative results across different functions and algorithms. Furthermore, we collect the selected $\{h_k\}_{k=1}^N$ over an experiment and plot the distribution of those rolling horizons in Figure \ref{fig:rolling}.


Based on Table \ref{result} and Figure \ref{fig:rolling}, we can obtain some important insights. First, the results indicate that our model clearly outperforms the state-of-the-art methods including non-myopic algorithms. The average and median Gaps of our algorithm are above 0.85, indicating that the estimations are improved 85$\%$ compared to the initial iteration. The key reason is that GLASSES and M-EI only consider fixed rolling horizon $h$, which is risky: the error propagation might eliminate the benefits of looking ahead. Indeed, choosing $h$ stagewise allows us to carefully avoid the negative effect of model mis-specification. As shown in Table \ref{result_h}, when we choose fixed rolling horizon at each stage, the resulting Gap will be affected. When $h=4,5$, the non-greedy algorithm will even sabotage the performance. Here we note that we believe a dynamic rolling horizon can also improve the performance of GLASSES and M-EI. However, this requires further analysis and theoretical inquiries.

Second, non-myopic algorithms are capable of beating greedy algorithms. Interestingly, the feasible rolling horizon $h^*$ is usually not large (Figure \ref{fig:rolling}). This result is encouraging as it implies that the computational burden does not need to increase significantly since a short horizon is most beneficial. Therefore, it is over-pessimistic to discard non-myopia if one is afraid of error accumulation and computational complexity. 

Lastly, the results indicate that the benefits of our method become increasingly significant for the high dimensional scenarios. This is intuitively understandable, due to ability of the non-greedy algorithm to efficiently explore the horizon. 

\begin{figure}[h]
\vspace{-0.1in}
\includegraphics[width=80mm, height=60mm, scale=0.8]{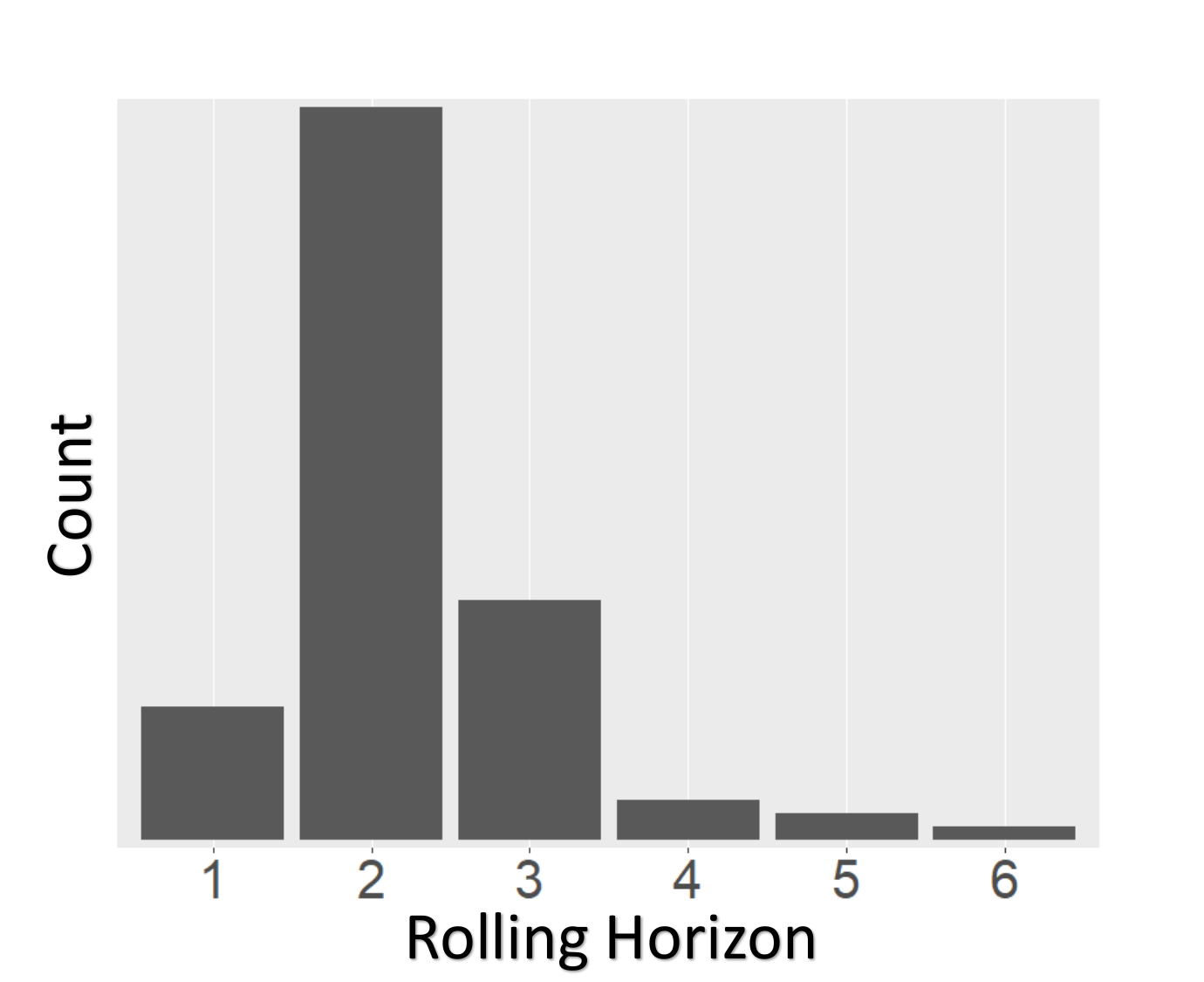}
\vspace{.1in}
\caption{Distribution of Rolling Horizon}
\label{fig:rolling}
\vspace{-0.1in}
\end{figure}




\subsubsection{Discount Factor}
We study the effect of different discount factors. Specifically, we choose $\alpha$ from set $\{0.6, 0.7, 0.8, 0.9\}$. The discount factor plays a role in ceiling the value of the rolling horizon as shown in \textbf{Theorem \ref{thm:h}}. An extreme case is when $\alpha=0$, the reward is collected immediately (i.e., greedily). Based on Table \ref{result} and \ref{result_alpha}, it seems that when $\alpha\in\{0.8, 0.9\}$, the performance is promising. This result is intuitive as a moderate discount factor encourages an algorithm to consider collecting future reward and is capable of generating improving results. 

\vspace{-0.65em}

\begin{table}[!htb]
  \caption{Mean Gap $G$ with respect to different fixed rolling horizon over 30 experiments with different initial points.}
  \label{result_h}
  \centering
  \scriptsize
  \begin{tabular}{ccccc}
    \toprule
    Function Name & $h=2$ & $h=3$ & $h=4$ & $h=5$ \\
    \midrule
    Branin-Hoo & 0.830 & 0.805 & 0.777 & 0.700 \\
    Six-Hump Camel & 0.855 & 0.860 & 0.671 & 0.665 \\
    Goldstein-Price & 0.829 & 0.824 & 0.732 & 0.667 \\
    Bohachevsky & 0.865 & 0.788 & 0.721 & 0.648 \\
    Griewant-3 & 0.802 & 0.755 & 0.621 & 0.683 \\ 
    \bottomrule
  \end{tabular}
\end{table}

\begin{table}[!htb]
  \caption{Mean Gap $G$ with respect to different discount factor over 30 experiments with different initial points.}
  \label{result_alpha}
  \centering
  \scriptsize
  \begin{tabular}{ccccc}
    \toprule
    Function Name & $\alpha=0.6$ & $\alpha=0.7$ & $\alpha=0.8$ & $\alpha=0.9$ \\
    \midrule
    Branin-Hoo & 0.812 & 0.801 & 0.867 & 0.864 \\
    Six-Hump Camel & 0.780 & 0.810 & 0.871 & 0.870 \\
    Goldstein-Price & 0.826 & 0.804 & 0.844 & 0.867 \\
    Bohachevsky & 0.803 & 0.818 & 0.853 & 0.872 \\
    Griewant-3 & 0.764 & 0.830 & 0.845 & 0.861 \\ 
    \bottomrule
  \end{tabular}
\end{table}


\section{Literature Review}
\label{sec:literature}
\subsection{Nonmyopia}
Few literature has focused on the non-myopic BO. \cite{ginsbourger2010towards} propose an expectation improvement (EI) criterion to derive sequential sampling strategies using Monte-Carlo simulation. Later, some approximation algorithms have been proposed that provide theoretical guarantees when sampling spaces are finite \citep{marchant2014sequential, ling2016gaussian}. Unfortunately such algorithms scale poorly with the number of rolling horizon considered. Later, \cite{gonzalez2016glasses} provided the GLASSES algorithm that relieves the myopia assumption of BO and can efficiently tackle an uncountable sampling space. GLASSES utilizes the long-sight loss function in \cite{osborne2010bayesian} and then propose an efficient optimization-marginalization scheme to solve that loss. Despite its strength, this approach assumes that the objective function is $L$-Lipschitz continuous. Besides the aforementioned methods, there exists some efficient multi-step look-ahead algorithms in the area of Bayesian feasibility determination and root-finding problems \citep{waeber2013bisection, cashore2016multi}. Nevertheless, they are only applicable to a very specific physical setting and cannot be easily generalized to a general framework. More Recently, \cite{lam2016bayesian, lam2017lookahead} proposed a look-ahead DP formulation using EI as a heuristic reward function. A direct extension to this work includes using the modified-EI (M-EI) \citep{groot2010bayesian, lam2015multifidelity} instead of EI to handle multi-information sources. However, a crucial drawback of the M-EI is that its selects sampling point and query sources separately. This might lead to reduced accuracy as joint optimality is not considered. Recently, \cite{jian2019twosteps} has proposed a practical two-step lookahead BO algorithm. This is one successful example that illustrates the benefits of looking sightly ahead.

\subsection{Multi-information Source}
We provide a short review on misoBO for completeness. Multi-information source optimization was thoroughly studied by \cite{swersky2013multi}. The authors argue that auxiliary tasks can aid in solving some expensive optimization problems. \cite{swersky2013multi} utilize a multivariate Gaussian process $\mathcal{GP}$ \citep{seeger2005semiparametric, bonilla2008multi} to model uncertainties in the objective function and predictive entropy search to decide on the next sampling location. Very recently, \cite{poloczek2017multi} improved the misoBO algorithm through utilizing a more flexible $\mathcal{GP}$ construction, using the linear model of coregionalization, and extending the KG algorithm to the setting with multiple information sources. They showed that the improved method (denoted as misoKG) can find sampling locations with higher value at reduced cost. Despite this seminal work, the misoKG does not consider far horizon planning since it uses a one-step look-ahead approach that only considers reducing regret at the next step. Besides misoBO, other closely related work belong to the problem of multi-fidelity optimization \citep{mcleod2017practical, kandasamy2016gaussian, cutajar2019deep}. These models have been mainly based on hierarchical model structures that restrict the information to be shared from low fidelity models. Also, they implement a myopic approach and fail to account for the future information such as remaining budget. 

\section*{Conclusion}
We provide a theoretical proof of the ``improving" nature of the rollout DP algorithm and a practical guideline on choosing a sequence of rolling horizons. We argue that rollout with a well chosen rolling horizon is beneficial in the sense that the error propagation is not catastrophic and the profits from the rollout improving nature remain. Therefore, the rollout DP has great promise in BO theory and applications. One possible future work is to generalize our analysis and apply it to other non-myopic methods. We hope our work will help inspire continued exploration into the non-myopic algorithms.

\newpage 
\section*{Appendix}

\section{Formulation}
In misoBO scenario, we have access to several sampling sources and we are interested in deciding both optimal sampling points and sampling sources.

\subsection{Setting}
We want to solve the unconstrained optimization problem $\bm{x}^*=\argmax_{\bm{x}\in\bm{\mathcal{X}}} f(\bm{x})$. Due to limited budget, sampling from the original source is expensive and incurs a cost $c(\bm{x}):\bm{\mathcal{X}}\to\mathbb{R^+} $. Now suppose we have access to $I$ possibly biased auxiliary sources indexed by $\mathcal{I}=\{1,\ldots,I\}$. Each source has a query cost $c_i(\bm{x}),i\in\mathcal{I}$. When sampling from source $i\in\mathcal{I}$ at point $\bm{x}$, we observe a noisy and biased outcome $y(i,\bm{x})$. We assume the observation $y(i,\bm{x})$ is normally distributed with mean $f(i,\bm{x})$ and variance $\sigma^2_i(\bm{x})$. Denote by $\delta_i(\bm{x}):\bm{\mathcal{X}}\to\mathbb{R}$ the bias term and $\delta_i(\bm{x})=f(i,\bm{x})-f(\bm{x})$ from each auxiliary source $i \in \mathcal{I}$. We set $\delta_i\sim\mathcal{GP}(0,\Sigma_i(\bm{x},\bm{x}'))$ and $f(\bm{x})\sim\mathcal{GP}(\mu_0(\bm{x}),\Sigma_0(\bm{x},\bm{x}'))$. Therefore, $f(i,\bm{x})$ is a GP with mean function $\mu(i,\bm{x})$ and covariance function $\Sigma((i,\bm{x}),(i',\bm{x}'))$. Specifically,
$\mu(i,\bm{x})=\mu_0(\bm{x}),\Sigma((i,\bm{x}),(i',\bm{x}'))=\Sigma_0(\bm{x},\bm{x}')+\mathbb{I}(i,i')\Sigma_i(\bm{x},\bm{x}')$, where $\mathbb{I}(i,i')=1$ if $i=i'$. Here we note that a mean function (or a constant) can be added to model systematic discrepancy in the bias $\delta_i$ \citep{higdon2008computer}.

Given data $\bm{D}_k=\{\bm{x}_1,y_1,i_1,\ldots,\bm{x}_k,y_k,i_k\}$, we would like to determine the next sampling duplet $(i_{k+1},\bm{x}_{k+1})$ by solving the following optimization problem: $(i_{k+1},\bm{x}_{k+1})\coloneqq(i^*,\bm{x}^*)=\argmax_{(i,\bm{x})\in(\mathcal{I},\bm{\mathcal{X}})} Q_k(i,\bm{x};\bm{D}_k)$. After observing the optimal sampling duplet, we augment the current training data $\bm{D}_k$ with the new observation and obtain $\bm{D}_{k+1}=\bm{D}_k\cup\{(\bm{x}_{k+1},y_{k+1},i_{k+1})\}$.

\subsection{Dynamic Programming}
Denote by $k\in\{1,..., N\}$. At each stage $k$, define the state space as $\mathcal{S}_k=(\bm{\mathcal{X}}\times\mathbb{R}\times\mathcal{I})$ and denote by dataset $\bm{D}_k\coloneqq s_k\in\mathcal{S}_k$ the current state, where $s_k$ is the potential state in the state space $\mathcal{S}_k$. A policy $\bm{\pi}=\{\pi_1,\ldots,\pi_N\}$ is a sequence of rules $\pi_k$ mapping the state space $\mathcal{S}_k$ to the design space $\bm{\mathcal{X}}$ and sources $\mathcal{I}$. We use $\pi^{\bm{\pi}}_k$ to emphasize the $k^{th}$ rule under policy $\bm{\pi}$. Let $\pi_k(\bm{D}_k)=(\bm{x}_{k+1},i_{k+1})$. Now denote by $r_k:\mathcal{S}_k\times\bm{\mathcal{X}}\times\mathcal{I}\to\mathbb{R}$ the reward function at stage $k$. Define the end-stage reward as $r_{N+1}:\mathcal{S}_{N+1}\to\mathbb{R}$. The discounted expected cumulative reward of a finite $N$-step horizon under policy $\bm{\pi}$ given initial dataset $\bm{D}_1$ can be expressed as $R^{\bm{\pi}}(\bm{D}_1)=$
\begin{equation}
\label{3:1}
    \mathbb{E}\bigg[\sum_{k=1}^{N} \alpha^{k-1} r_k(\bm{D}_k,\bm{x}_{k+1},i_{k+1})+\alpha^N r_{N+1}(\bm{D}_{N+1})\bigg].
\end{equation}
In the policy space $\bm{\Pi}$, we are interested in the optimal policy $\bm{\pi}^*\in\bm{\Pi}$ which maximizes Eq. \eqref{3:1}. Specifically,
\begin{equation}
    R^{\bm{\pi}^*}(\bm{D}_1)\coloneqq\max_{\bm{\pi}\in\bm{\Pi}}R^{\bm{\pi}}(\bm{D}_1).
\end{equation}
Based on the Bellman optimality equation, we can formulate \eqref{3:1} as a recursive DP: $R_k(\bm{D}_k)=$
\begin{equation}
\begin{split}
&\max_{(i_{k+1},\bm{x}_{k+1})\in(\mathcal{I},\bm{\mathcal{X}})}\mathbb{E}[r_k(\bm{D}_k,\bm{x}_{k+1},i_{k+1})+\alpha R_{k+1}(\bm{D}_{k+1})],
\end{split}
\end{equation}
with $R_{N+1}(\bm{D}_{N+1})=r_{N+1}(\bm{D}_{N+1})$. Therefore, the acquisition function is expressed as $Q_k(i_{k+1},\bm{x}_{k+1};\bm{D}_k)=$
\begin{equation}
\label{3:Q}
    \mathbb{E}\big[r_k(\bm{D}_k,\bm{x}_{k+1},i_{k+1})+\alpha R_{k+1}(\bm{D}_{k+1}) \big].
\end{equation}

\subsection{Knowledge Gradient}
The reward function at each stage $k$ quantifies the gains of applying rule $\pi_k$ given state $\bm{D}_k$. To handle multi-information source BO efficiently, we will adopt a normalized KG as our expected stage-reward function \citep{ryzhov2012knowledge, poloczek2017multi}. Specifically, $\mathbb{E}[r_k(\bm{D}_k,\bm{x}_{k+1},i_{k+1})]=$
\begin{equation}
\label{3:1:1}
    \mathbb{E}\big[\frac{1}{c_{i_{k+1}}(\bm{x}_{k+1})}\big(\max_{\bm{x}'}\mu^{k+1}(0,\bm{x}')-\max_{\bm{x}'}\mu^k(0,\bm{x}')\big)\big].
\end{equation}
The first part in the expected KG can be expressed as $\mathbb{E}\big[\max_{\bm{x}'}\mu^{k+1}(0,\bm{x}')\big]=$
\begin{equation}
\label{3:1:2}
    \mathbb{E}\big[\max_{\bm{x}'}\{\mu^k(0,\bm{x}')+\sigma_{\bm{x}'}^k(i,\bm{x}_{k+1})Z\}\big],
\end{equation}
where $Z$ is a standard normal random variable and $\sigma_{\bm{x}'}^k(i_{k+1},\bm{x}_{k+1})=$
\begin{equation*}
    \frac{\Sigma^k((0,\bm{x}'),(i_{k+1},\bm{x}_{k+1}))}{[\sigma^2_{i_{k+1}}(\bm{x}_{k+1})+\Sigma^k((i_{k+1},\bm{x}_{k+1}),(i_{k+1},\bm{x}_{k+1}))]^{1/2}},
\end{equation*}
such that $\Sigma^k$ is the posterior covariance function of $f$ given current data $\bm{D}_k$. Since we are taking expectation with respect to Gaussian random variables, equations \eqref{3:1:1} and \eqref{3:1:2} are easy to compute and can be efficiently estimated by a Gauss-Hermite quadrature with $n$ nodes. Under the single information source scenario, we simply let $I=1$. We summarize  our misoKG algorithm in \textbf{Algorithm 1}. 




\section{Algorithm}
The algorithm for the multi-information source BO is lised in $\textbf{Algorithm 1}$.
\begin{algorithm}[!htb]
    \label{algo:1}
	\SetAlgoLined
	\KwData{Initial data $\bm{D}_1$, budget $B$ and query cost $c,c_i$, number of remaining evaluations $N$.}
	\KwResult{Data $\bm{D}_N$, optimal value $f^{\bm{D}_{N}}_{max}$, Gap $G$.}
	Fit $\mathcal{GP}$ to data $\bm{D}_1$ and obtain parameters of bias terms and initial optimal value $f^{\bm{D}_1}_{max}$\;
	\For{$k=1:N$}{
		\eIf{ $B-\min_i c_i < 0$ }{
			Directly return $\bm{D}_k$ as $\bm{D}_N$\;
			STOP\;
		}{
			Choose feasible horizon $h$\;
			Select $(i_{k+1},\bm{x}_{k+1})=\argmax_{i\in\mathcal{I},\bm{x}\in\bm{\mathcal{X}}} \tilde{Q}_k(i,\bm{x};\bm{D}_k)$ s.t. $c_{i_{k+1}}(\bm{x}_{k+1})\leq B$\;
			$B$ $\gets$ $B-c_{i_{k+1}}(\bm{x}_{k+1})$\;
		}
		Evaluate $f(i_{k+1},\cdot)$ at $\bm{x}_{k+1}$ and obtain $y_{k+1}$\;
		Augment the dataset $\bm{D}_{k+1}=\bm{D}_k\cup\{(\bm{x}_{k+1},y_{k+1},i_{k+1})\}$\;
		Fit $\mathcal{GP}$ to data $\bm{D}_{k+1}$\;
		$k\gets k+1$\;
	}
	Fit $\mathcal{GP}$ to data $\bm{D}_{N}$\;
	Obtain optimal value $f^{\bm{D}_{N}}_{max}$\;
	Calculate the Gap $G$\;
	Return $\bm{D}_N$, $f^{\bm{D}_{N}}_{max}$ and $G$.
	\caption{The Non-myopic Multi-Information Source Bayesian Optimization Algorithm}
\end{algorithm}

\section{Performance Guarantees}
Under the multi-information source setting, the heuristic KG is also sequentially consistent and sequentially improving. 
\begin{coro}
    The KG algorithm is sequentially consistent and sequentially improving.
\end{coro}
\begin{proof}
Remember that state $s_k$ is the dataset $\bm{D}_k$. Assume KG algorithm starts at a state $s_k$ (i.e., current dataset $\bm{D}_k$). At each iteration of KG, given a path $(\bm{D}_k, \bm{D}_{k+1}, \ldots, \bm{D}_m)$ and $\bm{D}_m$ is not the state at the end, the next state $\bm{D}_{m+1}$ is obtained by solving the acquisition function of KG and augmenting $\bm{D}_m$ with $(\bm{x}^*, y, i^*)$. If $\bm{D}_{m+1}$ is not the terminating state, the algorithm will start with the path $(\bm{D}_k, \bm{D}_{k+1}, \ldots, \bm{D}_m,\bm{D}_{m+1}$). Otherwise, the algorithm will terminate with state $\bm{D}_{m+1}$ and $N=m+1$. Therefore, KG is sequentially consistent.  

Let $(\bm{D}_1, \bm{D}_2, \ldots, \bm{D}_k,\ldots, \bm{D}_N)$ be the path generated by rollout starting from $\bm{D}_1$. Define $\sigma(s)$ as the sub $\sigma$-algebra generated by state $s$.  Since KG is sequentially consistent, we have
\begin{equation}
    \label{4:1:1}
    \begin{split}
        &\mathbb{E}\big[\sum_{\ell=k}^N r_\ell(s_\ell,\pi^{\bm{\pi}_{\mathcal{H}(s)}}_\ell(s_\ell))|\sigma(s')\big]\\
        &=\mathbb{E}\big[\sum_{\ell=k}^N r_\ell(s_\ell,\pi^{\bm{\pi}_{\mathcal{H}(s')}}_\ell(s_\ell))|\sigma(s')\big],
    \end{split}
\end{equation}
where $s'$ is the subsequent state of $s$. Therefore, the rollout is sequentially improving and we complete our proof.
\end{proof}

\newpage

\subsubsection*{Reference}
\renewcommand{\bibname}{\vspace{-0.2in}} 
\bibliography{mybib}

\end{document}